\documentclass[10pt,twocolumn,letterpaper]{article}

\usepackage{iccv}              

\definecolor{iccvblue}{rgb}{0.21,0.49,0.74}
\definecolor{green}{rgb}{0.18,0.55,0.34}
\usepackage[pagebackref,breaklinks,colorlinks,allcolors=green]{hyperref}

\usepackage{pgfplots}

\usepackage{tikz}
\usepackage{xcolor}
\definecolor{blue}{RGB}{60,132,196}
\definecolor{red}{RGB}{207,78,56}
\definecolor{gray}{RGB}{146,146,161}
\definecolor{green4}{RGB}{46, 139, 87}
\definecolor{red2}{RGB}{149,9,30}
\usepackage{multirow} 
\usepackage{colortbl}
\definecolor{lightgray}{gray}{0.9}
\usepackage{hhline} 
\usepackage{pdflscape}

\usepackage{amssymb}
\usepackage{amsfonts}
\usepackage{amsthm}
\usepackage{amsmath}

\newtheorem{theorem}{Theorem}
\usepackage{graphicx}
\usepackage{pifont}

\usepackage{dsfont}

\usepackage{svg}
\usepackage{mathrsfs}
\DeclareMathAlphabet{\mathpzc}{OT1}{pzc}{m}{it}
\definecolor{pink1}{RGB}{239 41 140}
\definecolor{pink}{RGB}{30,144,255}
\definecolor{BarrierColor}{RGB}{25, 162, 30}
\definecolor{BicycleColor}{RGB}{79, 241, 199}
\definecolor{BusColor}{RGB}{198, 1, 90}
\definecolor{CarColor}{RGB}{150, 232, 149}
\definecolor{ConstVehColor}{RGB}{147, 83, 113}
\definecolor{MotorcycleColor}{RGB}{143, 43, 251}
\definecolor{PedestrianColor}{RGB}{9, 41, 148}
\definecolor{TrafficConeColor}{RGB}{86, 155, 103}
\definecolor{TrailerColor}{RGB}{118, 122, 211}
\definecolor{TruckColor}{RGB}{46, 6, 61}
\definecolor{DrivableSurfColor}{RGB}{185, 217, 16}
\definecolor{OtherFlatColor}{RGB}{142, 44, 241}
\definecolor{SidewalkColor}{RGB}{242, 218, 92}
\definecolor{TerrainColor}{RGB}{225, 12, 72}
\definecolor{ManmadeColor}{RGB}{111, 209, 144}
\definecolor{VegetationColor}{RGB}{182, 34, 111}

\providecommand{\ie}{\textit{\ie}}
\providecommand{\eg}{\textit{\eg}}
\providecommand{\vs}{\textit{v.s.}}
\providecommand{\wrt}{\it{w.r.t.}}

\pgfplotsset{compat=1.18}

\usepackage{marvosym}
\usepackage{ifsym}
\def\paperID{2727} 
\def\confName{ICCV}
\def\confYear{2025}

\title{Semantic Causality-Aware Vision-Based 3D Occupancy Prediction} 

\author{
Dubing Chen$^{1}$,
Huan Zheng$^{1}$,
Yucheng Zhou$^1$, \\
Xianfei Li$^2$, 
Wenlong Liao$^2$, 
Tao He$^2$, 
Pai Peng$^2$,
Jianbing Shen$^{1}$\textsuperscript{\Letter}\\
$^1$SKL-IOTSC, CIS, University of Macau \hspace{4mm}    $^2$COWAROBOT Co. Ltd. 
\\
\href{https://github.com/cdb342/CausalOcc}{\textcolor{pink1}{\texttt{https://github.com/cdb342/CausalOcc}}
}
}

\begin{document}
\maketitle

\begin{abstract}
\let\thefootnote\relax\footnotetext{\Letter \ Corresponding author: \textit{Jianbing Shen}. This work was supported in part by the Science and Technology Development Fund of Macau SAR (FDCT) under grants 0102/2023/RIA2 and 0154/2022/A3 and
001/2024/SKL and CG2025-IOTSC, the University of Macau SRG2022-00023-IOTSC grant, and the Jiangyin Hi-tech Industrial Development Zone under the Taihu Innovation Scheme (EF2025-00003-SKL-IOTSC).}

Vision-based 3D semantic occupancy prediction is a critical task in 3D vision that integrates volumetric 3D reconstruction with semantic understanding. Existing methods, however, often rely on modular pipelines. These modules are typically optimized independently or use pre-configured inputs, leading to cascading errors. 
In this paper, we address this limitation by designing a novel causal loss that enables holistic, end-to-end supervision of the modular 2D-to-3D transformation pipeline. Grounded in the principle of 2D-to-3D semantic causality, this loss regulates the gradient flow from 3D voxel representations back to the 2D features. Consequently, it renders the entire pipeline differentiable, unifying the learning process and making previously non-trainable components fully learnable. 
Building on this principle, we propose the Semantic Causality-Aware 2D-to-3D Transformation, which comprises three components guided by our causal loss: Channel-Grouped Lifting for adaptive semantic mapping, Learnable Camera Offsets for enhanced robustness against camera perturbations, and Normalized Convolution for effective feature propagation. Extensive experiments demonstrate that our method achieves state-of-the-art performance on the Occ3D benchmark, demonstrating significant robustness to camera perturbations and improved 2D-to-3D semantic consistency.
\end{abstract}

\begin{figure}[t]
	\centering
    \subfloat[An Example of Inaccurate 2D-to-3D Transformation]{%
        \includegraphics[width=0.47\textwidth]{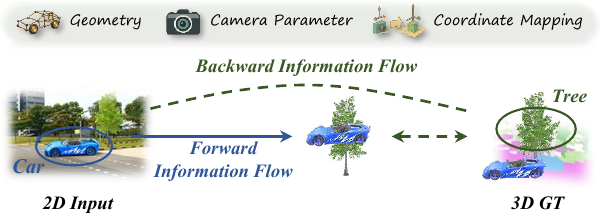}
        \label{fig:motivation_a}
    }\\
    \subfloat[Modular 2D-to-3D Transformation]{%
        \includegraphics[width=0.47\textwidth]{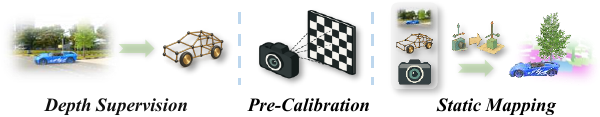}
        \label{fig:motivation_b}
    } \\
    
    \subfloat[Our End-to-End Supervised 2D-to-3D Transformation]{%
        \includegraphics[width=0.47\textwidth]{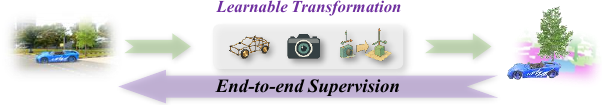}
        \label{fig:motivation_c}
    } 
    
	\caption{\textbf{(a) Illustrates a Visual Analysis of Semantic Ambiguity in VisionOcc.} Inaccurate 2D-to-3D transformations may lead to positional shifts, misaligning supervision signals and resulting in semantic ambiguity.
\textbf{(b) Depicts the Conventional Modular 2D-to-3D Transformation Paradigm }\cite{huang2021bevdet,ma2023cotr,li2023fbocc,chen2024alocc}, which employs depth supervision for geometry estimation, pre-calibrated camera parameters, and fixed mapping for lifting.
\textbf{(c) Presents Our Holistic, End-to-End Supervised 2D-to-3D Transformation Paradigm}, which eliminates the need for separate modular supervision or pre-calibration, enabling unified error propagation to supervise all components.}
	\label{fig:causal_analysis}
 \vspace{-4mm}
\end{figure}
\section{Introduction}
Predicting dense 3D semantic occupancy is a fundamental task in 3D vision, providing a fine-grained voxel representation of scene geometry and semantics \cite{yang2017semantic, tong2023scene, wang2023openoccupancy, tian2024occ3d}. The challenge of performing this prediction from vision alone, an approach known as vision-based 3D semantic occupancy prediction (VisionOcc), has recently become a focal point of research. By leveraging only commodity cameras, VisionOcc is pivotal for a wide range of 3D applications, serving as a comprehensive digital replica of the environment for tasks like analysis, simulation, and interactive visualization \cite{huang2023tri,tong2023scene,wei2023surroundocc,tian2024occ3d,li2023fbocc,chen2024alocc,chen2025rethinking}.

VisionOcc unifies the challenges of feed-forward 3D reconstruction and dense semantic understanding. Existing pipelines typically decompose this task into two meta-phases~\cite{huang2023tri,li2023fbocc,ma2023cotr,liu2024surroundsdf,chen2024alocc}. The initial 2D-to-3D transformation uses large receptive field operators like view lifting~\cite{philion2020lift,chen2024alocc} or cross attention~\cite{huang2023tri,li2023fbocc} to construct an initial 3D feature volume. Subsequently, a 3D representation learning phase employs operators with a small receptive field (\eg, 3D convolutions or local self-attention) to refine 3D features and produce the final prediction. Our work targets the 2D-to-3D transformation phase, which is more critical and error-prone. \cref{fig:motivation_a} showcases a primary failure mode where features of one class (\eg, a 2D `car') are erroneously transformed to the 3D location of another (\eg, a `tree'). This creates a flawed learning objective, forcing the model to learn a spurious association between `car' features and a `tree' label. Such \textit{semantic ambiguity} is a principal obstacle to achieving high performance.

The prevailing VisionOcc methods, typically based on Lift-Splat-Shoot (LSS), employ a modular approach for the 2D-to-3D transformation (\cref{fig:motivation_b})~\cite{philion2020lift,ma2023cotr,chen2024alocc}. This involves supervising geometry with a proxy depth loss while relying on fixed, pre-calibrated camera parameters and a static lifting map. However, this modularity raises critical questions about robustness and optimality. First, it is susceptible to compounding errors; for example, the reliance on fixed camera parameters makes the system vulnerable to real-world perturbations like camera jitter during motion. More fundamentally, the optimality of such proxy supervision is questionable. An intermediate representation ideal for depth estimation may not be optimal for the final semantic occupancy task, inherently limiting the transformation's expressive power due to this objective misalignment.
This motivates our central research question: \textbf{Can we devise an end-to-end supervision framework}\footnote{Here, ``end to end" refers to a unified supervision scheme for a process, distinct from the concept of a monolithic network architecture.} that holistically optimizes the entire 2D-to-3D transformation, enabling unified semantic-aware error backpropagation and allowing traditionally fixed modules to become fully learnable?

We approach this problem from a causal perspective. In VisionOcc, the 2D image semantics are the ``cause" of the final 3D semantic ``effect". Semantic misalignment arises from disrupted information flow from cause to effect (\cref{fig:motivation_a}). Therefore, instead of correcting the erroneous output, we propose to directly regularize the information flow itself. We posit that a 3D prediction for a given class should be influenced predominantly by 2D image regions of that same class. To enforce this information flow, we leverage gradients as a proxy, inspired by prior work~\cite{wang2023label,zhou2024visual,jiang2021layercam}. For each semantic class, the gradient of its aggregated 3D features is computed \textit{\wrt} the 2D feature map, producing a saliency-like map of 2D influence. This map is then directly supervised with the ground truth 2D segmentation mask. As shown in \cref{fig:motivation_c}, this establishes a principled, end-to-end supervision signal for the 2D-to-3D transformation, enabling holistic optimization of all its components.

To fully leverage our end-to-end supervision, we introduce a more expressive and learnable 2D-to-3D view transformation, termed the Semantic Causality-Aware Transformation (SCAT). A key challenge is that directly supervising gradients is inherently unstable. Therefore, the entire SCAT module is designed to constrain its gradient flow to a stable [0, 1] range. Specifically, SCAT introduces three targeted designs: \textit{i) Channel-Grouped Lifting:} To better disentangle semantics, we move beyond LSS's uniform weighting and apply distinct learnable weights to different groups of feature channels.
\textit{ii) Learnable Camera Offsets:} To mitigate motion-induced pose errors, we introduce learnable offsets to the camera parameters, which are implicitly supervised by the 2D-3D semantic consistency enforced by our causal loss.
\textit{iii) Normalized Convolution:} Finally, we employ a normalized convolution to densify the sparse 3D features from LSS~\cite{li2023fb}, ensuring this final step also adheres to our global gradient stability requirement.

Our contributions are as follows: \textbf{i)} We systematically analyze the 2D-to-3D transformation in VisionOCC, identifying a critical failure mode we term semantic ambiguity. We provide a theoretical analysis proving how the modularity of prior methods leads to error propagation, offering clear guidance for future work.
\textbf{ii)} To address these problems, we propose the Causal Loss that directly regularizes the information flow of the 2D-to-3D transformation. This enables true end-to-end optimization of all constituent modules, mitigating error accumulation and making previously fixed components, such as camera parameters, fully learnable.
\textbf{iii)} We instantiate our principles in the Semantic Causality-Aware Transformation, a novel 2D-to-3D transformation architecture. SCAT incorporates Channel-Grouped Lifting, Learnable Camera Offsets, and Normalized Convolution to explicitly tackle the challenges of semantic confusion, camera perturbations, and limited learnability.
\textbf{iv)} Extensive experiments show our method significantly boosts existing models, achieving a 3.2\% absolute mIoU gain on BEVDet. Furthermore, it demonstrates superior robustness to camera perturbations, reducing the relative performance drop on BEVDet from a severe -32.2\% to a mere -7.3\%.

\section{Related Work}
\subsection{Semantic Scene Completion}
Semantic Scene Completion (SSC) \cite{liu2018see,li2019rgbd,chen20203d,li2023voxformer,zhang2023occformer} refers to the task of simultaneously predicting both the occupancy and semantic labels of a scene.
Existing methods can be classified into indoor and outdoor approaches based on the scene type, with the former focusing on occupancy and semantic label prediction in controlled environments \cite{liu2018see, li2019rgbd, chen20203d}, while the latter shifts towards more complex outdoor settings, particularly in the context of autonomous driving \cite{behley2019semantickitti, cao2022monoscene}.
The core principle of SSC lies in its ability to infer the unseen, effectively bridging gaps in incomplete observations with accurate semantic understanding.
MonoScene \cite{cao2022monoscene} introduces a 3D SSC framework that infers dense geometry and semantics from a single monocular RGB image. 
VoxFormer \cite{li2023voxformer} presents a Transformer-based semantic scene completion framework that generates complete 3D volumetric semantics from 2D images by first predicting sparse visible voxel queries and then densifying them through self-attention with a masked autoencoder design.
OccFormer \cite{zhang2023occformer} introduces a dual-path transformer network for 3D semantic occupancy prediction, efficiently processing camera-generated 3D voxel features through local and global pathways, and enhancing occupancy decoding with preserve-pooling and class-guided sampling to address sparsity and class imbalance.

\subsection{Vision-based 3D Occupancy Prediction}
Vision-based 3D Occupancy Prediction \cite{tong2023scene,tian2024occ3d,wang2023openoccupancy,huang2023tri, liu2024fully, lu2023octreeocc,huang2024gaussianformer,huang2024selfocc,shi2024occupancy,chen2025rethinking} aims to predict the spatial and semantic features of 3D voxel grids surrounding an autonomous vehicle from image data. 
This task is closely related to SSC, emphasizing the importance of multi-perspective joint perception for effective autonomous navigation.
TPVFormer \cite{huang2023tri} is prior work that lifts image features into the 3D TPV space by leveraging an attention mechanism \cite{vaswani2017attention,li2022bevformer}. Different from TPVFormer, which relies on sparse point clouds for supervision, subsequent studies, including OccNet \cite{tong2023scene}, SurroundOcc \cite{wei2023surroundocc}, Occ3D \cite{tian2024occ3d}, and OpenOccupancy \cite{wang2023openoccupancy}, have developed denser occupancy annotations by incorporating temporal information or instance-level labels.
Methods such as BEVDet \cite{huang2021bevdet}, FBOcc \cite{li2023fbocc}, COTR \cite{ma2023cotr}, and ALOcc \cite{chen2024alocc} leverage depth-based LSS \cite{philion2020lift,li2022bevdepth,li2023bevstereo,li2024fast} for explicit geometric transformation, demonstrating strong performance.
Some methods \cite{pan2024renderocc, zhang2023occnerf,boeder2025gaussianflowocc,jiang2024gausstr,yan2024renderworld,zheng2024veon,boeder2024langocc} have explored rendering-based methods that utilize 2D signal supervision, thereby bypassing the need for 3D annotations.
Furthermore, recent research like \cite{tong2023scene, li2024viewformer, liu2024letoccflow,chen2024alocc,chen2024adaocc,liaocascadeflow,boeder2024occflownet} introduced 3D occupancy flow prediction, which addresses the movement of foreground objects in dynamic scenes by embedding 3D flow information to capture per-voxel dynamics. Unlike the above methods, we analyze the 2D-to-3D transformation process from the perspectives of error propagation and semantic causal consistency, proposing a novel approach that enhances causal consistency.

\begin{figure}[t]
    \centering 
    \setlength{\abovecaptionskip}{0pt}
    \includegraphics[width=0.99\linewidth]{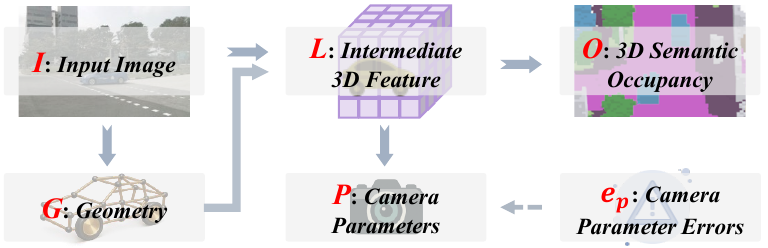}
    \vspace{1mm}
    \caption{\textbf{The Causal Structure of VisionOcc.} It illustrates the dependency chain from the image input $\textbf{I}$ to the semantic occupancy output $\textbf{O}$. $\textbf{G}$: geometry for 2D-to-3D transformation. $P$: camera intrinsic and extrinsic. $\textbf{L}$: intermediate 3D feature. $e_P$: errors in camera parameters.
    }
    \label{fig:causal_structure}
    \vspace{-4.4mm}
\end{figure}

\begin{figure*}[t]
    \centering 
    \includegraphics[width=0.78\linewidth]{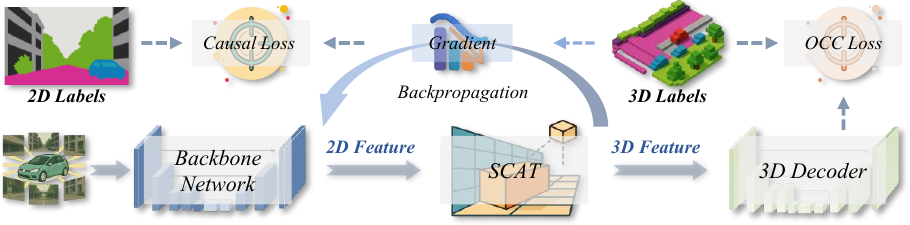}
    \caption{\textbf{The Overall Framework of the Proposed Semantic Causality-Aware VisionOcc.}
    The proposed framework consists of three primary components: a backbone network for extracting 2D features, an SCAT module for transforming these features into 3D space, and an Encoder-Decoder network for learning 3D semantics. The SCAT module is supervised by our causal loss.
    }
    \label{fig:framework}
    \vspace{-3mm}
\end{figure*}

\section{Method}
\label{sec:method}
\paragraph{Preliminary.}
VisionOcc predicts a dense 3D semantic occupancy grid from surround-view images, modeled as a causal dependency chain as shown in~\cref{fig:causal_structure}. This process involves key variables: input image \( \textbf{I} \), estimated geometry \( \textbf{G} \), camera parameters \( P \) (intrinsic \& extrinsic), potential camera parameter errors \( e_P \), intermediate 3D feature \( \textbf{L} \), and final occupancy output \( \textbf{O} \). In LSS-based methods \cite{huang2021bevdet,ma2023cotr,chen2024alocc}, the pipeline starts with the image backbone extracting features \( \mathbf{f}_i = F_i(\textbf{I}), \mathbf{f}_i \in \mathbb{R}^{U,V,C} \); geometry \(  \mathbf{G} \in \mathbb{R}^{U,V,D} \) is predicted as a probability distribution over discretized depth bins \( D \), \ie, $\mathbf{G} = F_g(\mathbf{f}_i)$; 2D features are then transformed to 3D via an outer product \( \mathbf{f}_L' = \mathbf{G} \otimes \mathbf{f}_i, \mathbf{f}_L' \in \mathbb{R}^{U,V,D,C} \); camera parameters \( P \) map these to voxel coordinates  \( R_P(u, v, d) \rightarrow (h, w, z) \in [0, H-1] \times [0, W-1] \times [0, Z-1] \), yielding \( \mathbf{f}_L \in \mathbb{R}^{H \times W \times Z }\), where \( H \times W \times Z \) defines the occupancy grid resolution; finally, \( \mathbf{L} \) is decoded to produce the semantic occupancy output: \(
\textbf{O} = F_o(\mathbf{f}_L), \textbf{O} \in \mathbb{R}^{H \times W \times Z \times S} 
\), where \( F_o \) is the decoding function and \( S \) is the number of semantic classes. The prediction $\mathbf{O}$ is supervised by the ground-truth $\mathbf{\tilde{O}}$. 

\subsection{Error Propagation in Depth-Based LSS} \label{sec:theo}
\begin{theorem}
\label{theorem:lss_limitation}
In Depth-Based LSS methods with a fixed 2D-to-3D mapping $M_{fixed}$, inherent mapping error $\delta M$ leads to gradient deviations, preventing convergence to an $\epsilon$-optimal solution. This is formalized as:
\begin{align}
    \!M_{fixed} = M_{ideal} + \delta M \!\implies\! \nabla_{\theta} L_{LSS} \neq \nabla_{\theta} L_{ideal},
\end{align}
where $M_{ideal}$ is the ideal mapping and $L_{ideal}$ is the loss function using $M_{ideal}$.
\end{theorem}

\begin{proof}
\textbf{Mapping Error Quantification:}
Let $\mathbf{x} \in \mathbb{R}^2$ be 2D pixel coordinates and $d(\mathbf{x})$ be the ground truth depth. Estimated depth is $\hat{d}(\mathbf{x}) = d(\mathbf{x}) + \epsilon_d(\mathbf{x})$, with $\epsilon_d(\mathbf{x})$ as depth error. Let $\mathbf{K}_{ideal} \in \mathbb{R}^{3 \times 3}$ be the ideal camera intrinsics, and $\mathbf{K} = \mathbf{K}_{ideal} + \epsilon_K$ be the estimated camera intrinsics with error $\epsilon_K$. The fixed mapping is $M_{fixed} = M_{ideal} + \delta M$, where $\delta M$ encompasses errors from various sources, including depth estimation error $\epsilon_d(\mathbf{x})$ and camera extrinsic error $\epsilon_K$. We assume the total mapping error is bounded $\|\delta M\|_F \leq \Delta_M < \infty$.  The 3D coordinates are:
\begin{align}
    \mathbf{X} &= M_{fixed}(\mathbf{x}, \hat{d}(\mathbf{x}), \mathbf{K}) \\
    \mathbf{X}_{ideal} &= M_{ideal}(\mathbf{x}, d(\mathbf{x}), \mathbf{K}_{ideal}) \\
    \Delta \mathbf{X} &= \mathbf{X} - \mathbf{X}_{ideal} = \delta M(\mathbf{x}, \hat{d}(\mathbf{x}), \mathbf{K}),
\end{align}
where $\|\Delta \mathbf{X}\|_2 \leq C \cdot \Delta_M$ for bounded inputs.

\textbf{Feature Space Deviation and Loss:}
Let $F_{2D}(\mathbf{x})$ be 2D features and $F_{3D}(\cdot) = Lift(F_{2D}(\mathbf{x}), \cdot)$. Assuming $Lift$ is $L_{Lift}$-Lipschitz continuous, the 3D feature deviation is:
\begin{align}
    \|F_{3D}(\mathbf{X}) - F_{3D}(\mathbf{X}_{ideal})\|_F &\leq L_{Lift} \|\mathbf{X} - \mathbf{X}_{ideal}\|_2 \\
    &\leq L_{Lift} \|\Delta \mathbf{X}\|_2 \leq \Delta_{F_{3D}}.\notag
\end{align}
The loss function is $L_{LSS} = \mathcal{L}(P_{3D}(\mathbf{X}), GT_{3D})$, where $P_{3D}(\mathbf{X}) = Seg_{3D}(F_{3D}(\mathbf{X}))$.

\textbf{Gradient Deviation and Optimization Limit:}
The gradient of $L_{LSS}$ \textit{\wrt} parameters $\theta$ is given by chain rule. However, since $M_{fixed}$ is fixed, $\frac{\partial \mathbf{X}}{\partial \theta} = 0$. Thus, the gradient becomes:
\begin{align}
    \nabla_{\theta} L_{LSS} = \frac{\partial L_{LSS}}{\partial P_{3D}} \frac{\partial P_{3D}}{\partial F_{3D}} \left( \frac{\partial F_{3D}}{\partial F_{2D}} \frac{\partial F_{2D}}{\partial \theta} \right).
\end{align}
Due to $\Delta_{F_{3D}} > 0$, the computed gradient $\nabla_{\theta} L_{LSS}$ is based on the deviated feature space, i.e.,
\begin{align}
    \nabla_{\theta} L_{LSS}(\theta) &= \nabla_{\theta} \mathcal{L}(P_{3D}(\mathbf{X}), GT_{3D}) \\
    &\neq \nabla_{\theta} \mathcal{L}(P_{3D}(\mathbf{X}_{ideal}), GT_{3D}) = \nabla_{\theta} L_{ideal}(\theta).\notag
\end{align}
The gradient deviation prevents mapping's direct optimization and limits convergence to an $\epsilon$-optimal solution.
\end{proof}
The theoretical analysis reveals that the inherent error in the fixed 2D-to-3D mapping of Depth-Based LSS methods fundamentally hinders gradient-based optimization from achieving optimal performance.
\begin{table}[t]
    \centering
    \small
    \setlength{\tabcolsep}{0.02\linewidth}
    \begin{tabular}{lccc}
        \toprule
        \textbf{Method} & \textbf{mIoU} & \textbf{mIoU\textsubscript{D}} & \textbf{IoU} \\
        \midrule
        Depth-Based LSS & 44.5 & 40.4 & 78.9 \\
        SCL-Aware LSS & 50.5\textcolor{red2}{$\uparrow$6.0} & 46.9\textcolor{red2}{$\uparrow$6.5} & 85.7\textcolor{red2}{$\uparrow$6.8} \\
        \bottomrule
    \end{tabular}
    \vspace{-2mm}
    \caption{\textbf{Performance of Depth-Based LSS \vs SCL-Aware LSS on Occ3D (in Ideal Conditions).} BEVDetOcc is the baseline.
    }
    \vspace{-5mm}
    \label{tab:logic}
\end{table}
\subsection{Semantic Causal Locality in VisionOcc}
As revealed in our theoretical analysis (\cref{sec:theo}), Depth-Based LSS methods suffer from inherent error propagation due to their fixed 2D-to-3D mapping, which limits optimization efficacy and potential performance. 
To overcome these limitations, we particularly focus on strengthening the semantic causality of the 2D-to-3D transformation.
We argue that VisionOcc’s 2D-to-3D semantic occupancy prediction should exhibit \textit{semantic causal locality (SCL)} for robust perception in autonomous driving. Ideally, 2D causes should drive 3D semantic effects.
For instance, a predicted “car” at 3D location \( (h, w, z) \) should originate from a matching 2D image region of “car”. Per the causal chain, camera parameters \( P \) and estimated geometry \(\mathbf{G} \) enable this dependency, with \(\mathbf{G} \) being crucial to maintain SCL.

Next, we formulate the ideal SCL condition. For a 2D pixel \((u, v)\) with semantic label \(s\), the projection probability \(p_d\) (the value of \(\mathbf{G}\) at this coordinate and depth \(d \in D\)) should be high if its corresponding 3D ground-truth semantic is \(s\), and low otherwise:
\[
p_d \propto \mathds{1}(\mathbf{\tilde{O}}(R_P(u, v, d) + e_P) = s),
\]
where \(\mathds{1}\) is the indicator function, and \(e_P\) represents the potential coordinate transformation error caused by factors such as camera pose error. In practice, \(p_d\) acts as a weight multiplied by 2D features (\cref{eq:p_d}), enabling a probabilistic mapping that supports differentiable backpropagation.

\noindent\textbf{Limitations of Depth-Based LSS in SCL.}
Depth-Based LSS does not fully account for semantic causal consistency. It only preserves semantic causal locality intuitively under ideal conditions where \( e_P = 0 \) and depth estimation is perfectly accurate. However, with coordinate transformation errors \( e_P \), even a high \( p_d \) for the ideal depth may project 2D semantics to incorrect 3D locations, \ie,
\[
\mathbf{\tilde{O}}(R_P(u, v, d) + 0) = s, ~ \text{but} ~ \mathbf{\tilde{O}}(R_P(u, v, d) + e_P) \neq s.
\]
This misalignment causes 2D semantics $s$ (\eg, “car”) to link with wrong 3D semantics (\eg, “tree”), causing semantic ambiguity and hindering training. Moreover, depth-based LSS often propagates semantics to surface points, weakening semantic propagation to occluded regions \cite{chen2024alocc}.

\noindent\textbf{Empirical Validation.}
We conduct an empirical study to validate our analysis of SCL, comparing two ideal geometric transformations. Using BEVDetOcc~\cite{huang2021bevdet} as the baseline, we replace its estimated depth-based geometry for LSS with: \textit{i)} ground-truth LiDAR depths; \textit{ii)} SCL-Aware geometry, which computes \( p_d \) from 2D and 3D semantic ground truths, where for \( (u, v, d, s) \), \( p_d = 1 \) if \( \mathbf{\tilde{O}}(R_P(u, v, d) ) = s \), else \( p_d = 0 \). We evaluate their performance in semantic occupancy prediction.
\cref{tab:logic} demonstrates that 2D-to-3D Transformation achieves significant performance improvements over depth-based LSS in ideal conditions, validating the benefits of semantic causal locality.

\noindent\textbf{Summary.}
Building on the above analysis, we propose our solution to enforce semantic causality constraints during training. The overall framework is shown in \cref{fig:framework}.
\subsection{Semantic Causality-Aware Causal Loss}
\label{sec:causal-loss}
For lifting methods like LSS, we could directly supervise the transformation geometry \( \mathbf{G} \) using the causal semantic geometry derived from ground-truth labels (as described in the previous section). However, we aim to enhance the lifting method in the next section, rendering direct supervision impractical. Thus, we design a gradient-based approach to enforce semantic causality.

We begin within the LSS framework. For a 2D pixel feature at location \( (u, v) \), LSS multiplies it by the depth-related transformation probability \( p_d \) and projects it to the 3D coordinate corresponding to depth \( d \):
\begin{equation}
\mathbf{f}_L(R_P(u, v, d)) = p_d(u, v, d) \cdot \mathbf{f}_i(u, v, d).
\label{eq:p_d}
\end{equation}
Here, \( \mathbf{f}_L(R_P(u, v, d)) \in \mathbb{R}^C \) represents the 3D voxel feature at the projected location, with \( e_P \) omitted for notational simplicity. \( \mathbf{f}_i(u, v, d) \in \mathbb{R}^C \) denotes the 2D image feature at location \( (u, v) \).
We backpropagate gradients from \( \mathbf{f}_L \) to \( \mathbf{f}_i \), \ie,
\begin{equation}
\frac{\partial \sum \mathbf{f}_L(R_P(u, v, d))}{\partial \mathbf{f}_i(u, v, d)} = p_d \cdot \mathbf{I},
\end{equation}
where \( \mathbf{I} \) is the all-ones vector.

For each semantic class \( s \), we aggregate the features \( \mathbf{f}_L \) at all 3D positions where the ground truth class equals \( s \). This aggregation is backpropagated to the 2D features \( \mathbf{f}_i \), yielding a gradient map \( \nabla_s \in \mathbb{R}^{U \times V \times C} \) for class \( s \):
\begin{equation}
\nabla_s(u, v, c) = \sum_{(h', w', z') \in \Omega_s} \frac{\partial \sum \mathbf{f}_L(h', w', z', c)}{\partial \mathbf{f}_i(u, v, c)},
\end{equation}
where \( \Omega_s = \{(h', w', z') \mid O(h', w', z') = s\} \) is the set of 3D positions with semantic label \( s \), and \( c \) indexes the feature channels. Averaging over channel \( C \) produces an attention map \( A_s \in \mathbb{R}^{U \times V} \):
\begin{equation}
A_s(u, v) = \frac{1}{C} \sum_{c=1}^C \nabla_s(u, v, c).
\end{equation}

Finally, we enforce per-pixel constraints using 2D ground-truth labels with a binary cross-entropy loss:
\begin{equation}
\begin{aligned}
L^s_{bce} =& -\frac{1}{U \cdot V} \sum_{u, v} \big[ Y_s(u, v) \log A_s(u, v)  \\
+&(1 - Y_s(u, v)) \log (1 - A_s(u, v)) \big],
\end{aligned}
\end{equation}
where \( Y_s(u, v) \in \{0, 1\} \) is the 2D ground-truth label for semantic class \( s \) at pixel \( (u, v) \).

The gradient computation can be performed using the automatic differentiation calculator like \textit{torch.autograd}, requiring \( S \) backward passes to iterate over the \( S \) semantic classes. This incurs a computational overhead scaling linearly with the class amount \( S \).
To mitigate this overhead, we reformulate the loss computations in terms of the expectation of an unbiased estimator \cite{hoffman2019robust}. We define the expected BCE loss across all semantic classes \( s \in \{1, \dots, S\} \) as:
\begin{equation}
\mathbb{E}[L^s_{bce}] = \frac{1}{S} \sum_{s=1}^{S} L^s_{bce}.
\end{equation}

Based on this relationship, we uniformly sample a single semantic class \( s \) during training:
\begin{equation}
L_{causal} = L^s_{bce}, \quad s \sim \text{Uniform}(1, S).
\label{eq:causal}
\end{equation}
This sampling preserves the unbiased nature of the gradient and loss estimates and reduces the computational cost to \( \frac{1}{S} \). \( L_{\text{causal}} \) focuses on enhancing the geometric transformation, serving as an auxiliary term to complement the occupancy loss (case-by-case, \eg, cross-entropy in BEVDet \cite{huang2021bevdet}).

\subsection{Semantic Causality-Aware Transformation}
\label{sec:semantic-causal-lifting}
We propose semantic causality-aware 2D-to-3D transformation to enhance 2D-to-3D lifting, as shown in \cref{fig:modules}. \cref{eq:causal} constrains the geometry \( \mathbf{G} \) using 2D and 3D semantics. This overcomes the rigid, per-location hard alignment of geometric probabilities (\eg, using LiDAR depth supervision) in prior methods \cite{li2022bevdepth,li2023bevstereo,huang2021bevdet,li2023fbocc,chen2024alocc}. It enables advanced lifting designs, addressing errors from camera pose and other distortions.

\subsubsection{Channel-Grouped Lifting}
\label{sec:channel-grouped-lifting}

Vanilla LSS applies uniform weights to all 2D feature channels. We argue this is trivial as 2D and 3D features have distinct locality biases. For instance, a 2D “car” edge may capture “tree” semantics via convolution, but in 3D, these objects are distant. Uniform weighting both semantics causes ambiguity. Since different channels typically encode distinct semantics, we group the feature channels and learn unique weights for each group:
\begin{equation}
\resizebox{.88\linewidth}{!}{$
\mathbf{f}_{L,g}(R_P(u, v, d)) = \omega_{g,d} \cdot \mathbf{f}_{i,g}(u, v, d), ~ g \in \{1, \dots, N_g\},
$}
\end{equation}
where \( \mathbf{f}_{L,g} \in \mathbb{R}^{C/N_g} \) and \( \mathbf{f}_{i,g} \in \mathbb{R}^{C/N_g} \) are the 3D and 2D features for group \( g \). \( \omega_{g,d} \) is the learned weight for group \( g \), replacing \( p_d \) which uniformly lifts all channels. \( N_g \) is the number of groups. This preserves semantic distinction, ensuring channel-specific causal alignment.

\begin{figure}[t]
	\centering
  
        \includegraphics[width=0.99\linewidth]{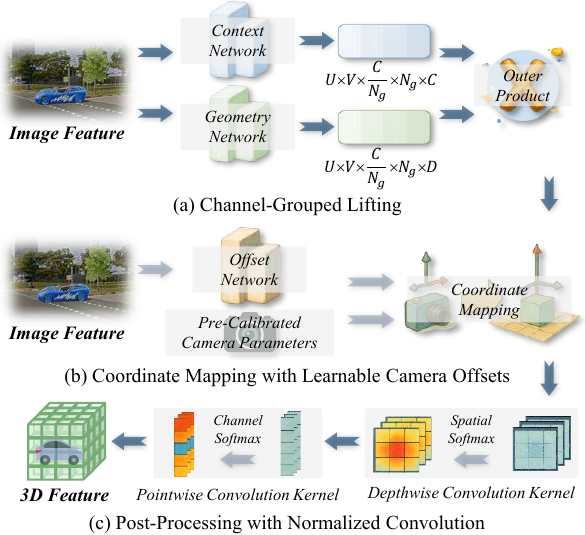}
  
	\caption{\textbf{The Detailed Structure of the Proposed Modules in Semantic Causality-Aware 2D-to-3D Transformation.}
    There are three key components, including (a): Channel-Grouped Lifting, (b): Learnable Camera Offsets, and (c): Normalized Convolution, enabling accurate and robust 2D-to-3D transformation.
 }
	\label{fig:modules}
 \vspace{-4mm}
\end{figure}

\subsubsection{Learnable Camera Offsets}
\label{sec:learnable-camera-parameters}
To address camera parameter errors, especially pose inaccuracies, we introduce learnable offsets into camera parameters. First, we ensure the lifting process is coordinate-differentiable. This is crucial for the offsets to receive gradients and adapt during training. The transformation of 2D image coordinates \( (u, v) \) and depth \( d \) to 3D voxel coordinates can be represented as matrix multiplication:
\begin{equation}
[h,w,z]^T = P \cdot [u \cdot d, v \cdot d, d, 1]^T,
\end{equation}
where \( P \in \mathbb{R}^{3 \times 4} \) is the camera projection matrix combining intrinsics and extrinsics. LSS typically rounds floating-point coordinates to integers, rendering them non-differentiable. Following ALOcc \citep{chen2024alocc}, we use the soft filling to enable differentiability \textit{\wrt} position. This method calculates distances between floating-point 3D coordinates and their eight surrounding integer coordinates. These distances serve as weights to distribute a 2D feature at \( (u, v, d) \) across multiple 3D locations. Lifting can be rewritten to
\begin{equation}
\begin{aligned}
\mathbf{f}_{L,g}(h',w&',z') = \omega_{g,d} \cdot \omega_{h',w',z'} \cdot \mathbf{f}_{i,g}(u, v, d),\\ &\forall (h',w',z') \in \text{neighbors},
\end{aligned}
\end{equation}
where \( \omega_{h',w',z'} \) are trilinear interpolation weights.

Next, we propose learning two offsets. First, we directly predict an offset applied to the camera parameters:
\begin{equation}
P:= P + \Delta P, \quad \Delta P = F_{offset1}(\mathbf{f}_i, P),
\end{equation}
where \( \Delta P \) is the predicted parameter offset, and $F_{offset1}$ denotes the network. Second, we estimate per-position offsets for each \( (u, v, d) \) in the image coordinate system:
\begin{equation}
\begin{aligned}
(u, v, d&)  := (u+\Delta u, v+\Delta v, d+\Delta d),\\
(\Delta u,& \Delta v, \Delta d)= F_{offset2}(\mathbf{f}_i(u, v, d)),
\end{aligned}
\end{equation}
where $F_{offset2}$ is another network. These offsets enable the model to adaptively compensate for camera parameter errors \(e_P\), improving geometric accuracy while preserving semantic causal locality under such errors.

\subsubsection{Normalized Convolution}
\label{sec:gradient-normalized-conv}

Prior work \cite{li2023fb} notes that the direct mapping in LSS yields sparse 3D features. To address this, an intuitive solution is to use local feature propagation operators (\eg, convolutions) with causal loss supervision for preserving semantic causality. However, vanilla convolutions lack gradient constraints during causal loss computation, producing poor results. We address this by normalizing convolution weights to keep gradient map values in [0, 1]. Given the difficulty of normalizing standard convolution weights, we follow MobileNet and ConvNext to adopt depthwise (spatial) and pointwise (channel) decomposition. For the depthwise kernel $W_{\text{spatial}} \in \mathbb{R}^{3 \times 3 \times 3 \times C}$, we apply softmax across spatial dimensions ($h, w, z$) per channel $c$:
\begin{equation}
\resizebox{.887\linewidth}{!}{$
\!W_{\text{spatial}}'[h,\! w,\! z,\! c]\!=\!\frac{\exp(W_{\text{spatial}}[h, w, z, c])}{\sum_{h'\!, w'\!, z'} \!\exp(W_{\text{spatial}}[h'\!, w'\!, z'\!, c])}\!.\!
$}
\end{equation}
For the pointwise kernel $W_{\text{channel}} \in \mathbb{R}^{C \times C}$, we apply softmax across input channels per output channel:
\begin{equation}
W_{\text{channel}}'[c_{\text{in}}, c_{\text{out}}] = \frac{\exp(W_{\text{channel}}[c_{\text{in}}, c_{\text{out}}])}{\sum_{c_{\text{out}}'} \exp(W_{\text{channel}}[c_{\text{in}}, c_{\text{out}}'])}.
\end{equation}
Specifically, we use transposed convolution for depthwise convolutions. It diffuses features from non-zero to zero positions, but the zero position does not affect others. We prove in the \textit{supplement} that the derived gradient mask remains within [0, 1] even with our semantic causality-aware 2D-to-3D transformation.

\begin{figure}[t]
    \centering
    \includegraphics[width=0.72\linewidth]{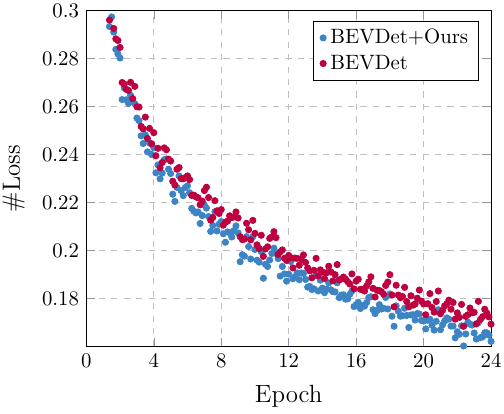}
    \vspace{-2mm}
    \caption{\textbf{Occupancy Loss Curves with/without Our Method.} Our method reduces training loss through semantic causal 2D-to-3D geometry transformation, as shown in the comparison before \textbf{(red)} and after \textbf{(blue)} its application.}
    \vspace{-2mm}
    \label{fig:loss_curve}
\end{figure}

\subsection{Validation of Gradient Error Mitigation}
\cref{fig:loss_curve} shows the occupancy loss curves during training for BEVDet, comparing performance with and without our method.  The results show that BEVDet integrated with our approach (blue line) achieves a significantly faster and steeper loss reduction compared to the original BEVDet (red line).
This empirical evidence validates our theoretical analysis of gradient error mitigation.  As Theorem \cref{theorem:lss_limitation} formalizes, Depth-Based LSS methods are inherently limited by gradient deviations due to fixed 2D-to-3D mapping errors.  In contrast, our method aims to alleviate this by enabling a learnable mapping and incorporating a causal loss. It indicates that by mitigating gradient error through semantic causality-aware 2D-to-3D transformation, our approach facilitates more efficient gradient-based learning, leading to faster convergence and a lower loss curve.

\begin{table}[t]
    \scriptsize
    \setlength{\abovecaptionskip}{0pt}
    \setlength{\tabcolsep}{0.01\linewidth}
    \begin{center}
        \begin{tabular}{l|cc|cc|cc}
            \toprule
            Method  & mIoU $\uparrow$ & Drop & mIoU\textsubscript{D}$\uparrow$ & Drop & IoU$\uparrow$ & Drop \\
            \midrule
            \midrule
            BEVDetOcc \cite{huang2022bevdet4d}  & 37.1 & \multirow{2}{*}{-32.3\%} & 30.2 & \multirow{2}{*}{-49.0\%} & 70.4 & \multirow{2}{*}{-8.5\%} \\
            BEVDetOcc \textit{w/ Noise} & 25.1 & & 15.4 & & 64.4 & \\
            \midrule
            \rowcolor{pink!8} \textbf{BEVDetOcc+Ours}   & 38.3 & \multirow{2}{*}{\textbf{-7.3\%}} & 31.5 & \multirow{2}{*}{\textbf{-10.8\%}} & 71.2 & \multirow{2}{*}{\textbf{-1.4\%}} \\
            \rowcolor{pink!8} \textbf{BEVDetOcc+Ours \textit{w/ Noise}} & 35.5 &\multirow{-2}{*}{\textbf{-7.3\%}}  & 28.1 &\multirow{-2}{*}{\textbf{-10.8\%}} & 70.2 &  \multirow{-2}{*}{\textbf{-1.4\%}} \\
            \midrule
            \midrule
            ALOcc \cite{chen2024alocc}  & 40.1 & \multirow{2}{*}{-21.9\%} & 34.3 & \multirow{2}{*}{-28.6\%} & 70.2 & \multirow{2}{*}{-8.4\%} \\
            ALOcc \textit{w/ Noise} & 31.3 & & 24.5 & & 64.3 & \\
            \midrule
            \rowcolor{pink!8} \textbf{ALOcc+Ours}   & 40.9 & \multirow{2}{*}{\textbf{-3.3\%}} & 35.5 & \multirow{2}{*}{\textbf{-4.8\%}} & 70.7 & \multirow{2}{*}{\textbf{-1.0\%}} \\
            \rowcolor{pink!8} \textbf{ALOcc+Ours \textit{w/ Noise}} & 39.6 & \multirow{-2}{*}{\textbf{-3.3\%}} & 33.8 & \multirow{-2}{*}{\textbf{-4.8\%}}& 70.0 &\multirow{-2}{*}{\textbf{-1.0\%}}  \\
            \bottomrule
        \end{tabular}
    \end{center}
    \vspace{-1mm}
    \caption{\textbf{Performance Comparison on the Occ3D Dataset with Gaussian Noise Added to Camera Parameters.} The ``Drop (\%)" columns show degradation, with our methods (BEVDetOcc+Ours, ALOcc+Ours) achieving much smaller drops (\eg, -7.3\% mIoU vs. -32.4\% for BEVDetOcc).}
    \vspace{-4.4mm}
\label{tab:add_noise}
\end{table}

\section{Experiment}
\subsection{Experimental Setup}
\textbf{Dataset.} 
In this study, we leverage the Occ3D-nuScenes dataset \cite{caesar2020nuscenes,tian2024occ3d}, a comprehensive dataset with diverse scenes for autonomous driving research. 
It encompasses 700 scenes for training, 150 for validation, and 150 for testing. 
Each scene integrates a 32-beam LiDAR point cloud alongside 6 RGB images, captured from multiple perspectives encircling the ego vehicle. 
Occ3D \cite{tian2024occ3d} introduces voxel-based annotations, covering a spatial extent of -40 m to 40 m along the X and Y axes, and -1 m to 5.4 m along the Z axis, with a consistent voxel size of 0.4 m across all dimensions. 
Occ3D delineates 18 semantic categories, comprising 17 distinct object classes plus an \textit{empty} class to signify unoccupied regions. 
Following \cite{tian2024occ3d,li2023fbocc,ma2023cotr,chen2024alocc}, we evaluate the occupancy prediction performance with mIoU across 17 semantic object categories, mIoU$_{D}$ for 8 dynamic categories, and occupied/unoccupied IoU for scene geometry.

\noindent\textbf{Implementation Details.}
We integrate our approach into BEVDet \cite{huang2022bevdet4d} and ALOcc \cite{chen2024alocc} for performance evaluation. The model parameters, image, and BEV augmentation strategies are retrained as the original. For ALOcc, we remove its ground-truth depth denoising module to adapt to our method. We use multiple convolutional layers to predict the two camera parameter offsets. The loss weight of \( L_{\text{causal}} \) is set to 0.02 in the main experiments. We optimize using AdamW \cite{loshchilov2018fixing} with a learning rate of \( 2 \times 10^{-4} \), a global batch size of 16, and for 24 epochs. Experiments use single-frame surrounding images, emphasizing improvements from enhanced 2D-to-3D transformations.

\begin{figure}[t]
	\centering
    \includegraphics[width=1.\linewidth]{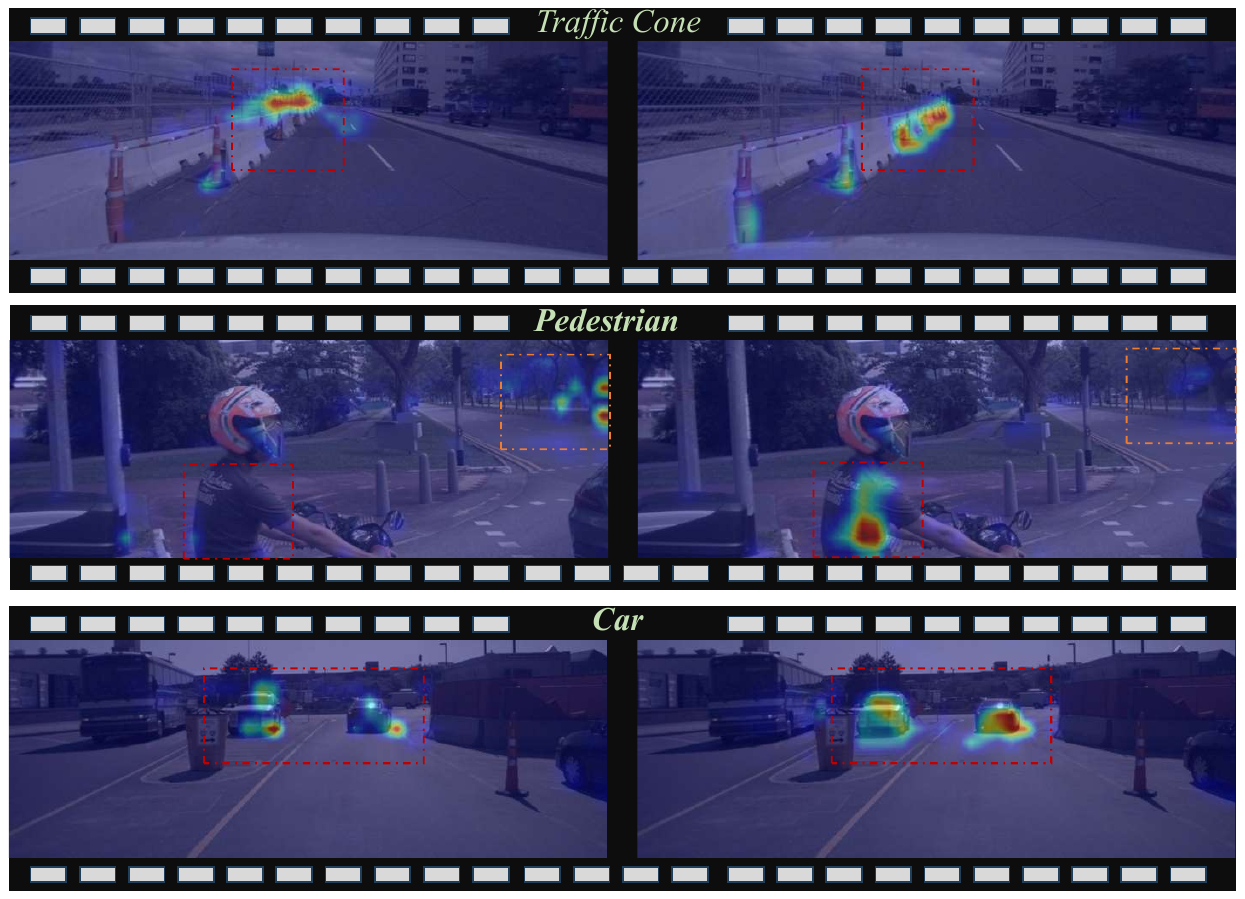}\\
    \vspace{-1mm}
    \begin{minipage}{0.49\linewidth}
        \centering
        \small
        (a) BEVDet
    \end{minipage}
    \hfill
    \begin{minipage}{0.49\linewidth}
        \centering
        \small
        (b) BEVDet + Ours
    \end{minipage}
    \vspace{-2mm}
	\caption{\textbf{Visualization of 2D-to-3D Semantic Causal Consistency Using LayerCAM \cite{jiang2021layercam}.} We enhance BEVDet with our method for comparison. Attention maps are computed for critical traffic classes: ``traffic cone", ``pedestrian", and ``car". Areas of greatest difference are marked with boxes. Each class-specific localization highlights our method’s precise focus over vanilla BEVDet, showing improved semantic alignment.}
	\label{fig:vis}
 \vspace{-4.4mm}
\end{figure}

\subsection{Evaluation of Camera Perturbation Robustness} 
To assess our method's robustness under extreme noise, we add Gaussian noise with (0.1 variances) to camera parameters in training and testing, as shown in ~\cref{tab:add_noise}. Compared to vanilla BEVDetOcc and ALOcc, our methods show smaller performance drops. BEVDetOcc+Ours has a mIoU drop of -7.3\% \vs -32.3\% for BEVDetOcc, and ALOcc+Ours drops -3.3\% \vs -21.9\% for ALOcc, demonstrating enhanced resilience to noisy parameters. This effectively counters motion-induced errors, benefiting self-driving tasks.

\subsection{Semantic Causality Visualization}
\label{sec:causality-vis}

As shown in the \cref{fig:vis}, we visualize 3D-to-2D semantic causal consistency using LayerCAM \cite{jiang2021layercam}. We backpropagate the final 3D semantic occupancy predictions of distinct classes to the 2D feature maps fed into SCAT. Notably, we are the first to apply LayerCAM for cross-dimensional analysis. \cref{fig:vis} (b) shows our method precisely focuses on class-specific locations. This confirms LayerCAM’s cross-dimensional effectiveness. Our approach surpasses vanilla BEVDet (\cref{fig:vis} (a)) in targeting class-associated objects. This proves improved semantic localization.

\subsection{Benchmarking with Previous Methods}
\label{sec:benchmark-comp}
As shown in \cref{tab:occ3d}, we compare our method with leading 3D semantic occupancy prediction approaches on Occ3D \cite{tian2024occ3d}. Specifically, compared to baseline models BEVDet and ALOcc, our method achieves significant improvements in mIoU, mIoU\textsubscript{D}, and IoU. For instance, the \textbf{BEVDetOcc+Ours} variant achieves an mIoU of \textbf{38.3}, surpassing BEVDetOcc~\cite{huang2022bevdet4d} by \textbf{1.2}, while improving mIoU\textsubscript{D} by \textbf{1.3} and IoU by \textbf{1.2}. Similarly, \textbf{ALOcc+Ours} shows gains of~\textbf{0.8} in mIoU, \textbf{1.1} in mIoU\textsubscript{D}, and \textbf{0.5} in IoU over ALOcc~\cite{chen2024alocc}. These results validate the superiority of our semantic causality-aware 2D-to-3D transformation.

\begin{table}[t]
    \footnotesize
    \setlength{\abovecaptionskip}{0pt}
    \setlength{\tabcolsep}{0.008\linewidth}
    \begin{center}
     \resizebox{0.475\textwidth}{!}{
        \begin{tabular}{l|c|c|c|c|c}
            \toprule
            Method &Backbone &Input Size& mIoU & mIoU\textsubscript{D}  & IoU \\
            \midrule
            \midrule
            MonoScene \cite{cao2022monoscene} &ResNet-101 &$928 \times 1600$ &6.1 & 5.4& - \\
            CTF-Occ \cite{tian2024occ3d} &ResNet-101&$928 \times 1600$&28.5 &27.4 &- \\
            TPVFormer \cite{huang2023tri}&ResNet-101&$928 \times 1600$&27.8 &27.2 &- \\

            COTR~\cite{ma2023cotr} &ResNet-50&$256 \times 704 $& 39.1 & 33.8 & 69.6 \\
            ProtoOcc~\cite{kim2024protoocc}  &ResNet-50&$256 \times 704$& 39.6 & 34.3 & - \\
            LightOcc-S~\cite{zhang2024lightweight}&ResNet-50  &$256 \times 704$& 37.9 & 32.4 & - \\
            DHD-S~\cite{wu2024deep} &ResNet-50 &$256 \times 704$& 36.5 & 30.7 & - \\
            FlashOCC~\cite{yu2023flashocc} &ResNet-50&$256 \times 704$& 32.0 & 24.7 & 65.3 \\
            FB-Occ \cite{li2023fbocc}  &ResNet-50&$256 \times 704$& 35.7 & 30.9 & 66.5 \\
    
            \midrule
            \midrule
            BEVDetOcc \cite{huang2022bevdet4d} &ResNet-50&$256 \times 704$ & 37.1 & 30.2 & 70.4 \\
            \rowcolor{pink!8} \textbf{BEVDetOcc+Ours} &ResNet-50 &$256 \times 704$ & \textbf{38.3{\color{red2}~$\uparrow$1.2}} & \textbf{31.5{\color{red2}~$\uparrow$1.3}}& \textbf{71.2{\color{red2}~$\uparrow$1.2}} \\
            \midrule
            ALOcc \cite{chen2024alocc}&ResNet-50 &$256 \times 704$ & 40.1 & 34.3 & 70.2 \\
            \rowcolor{pink!8} \textbf{ALOcc+Ours} &ResNet-50&$256 \times 704$  & \textbf{40.9{\color{red2}~$\uparrow$0.8}} & \textbf{35.5{\color{red2}~$\uparrow$1.1}}& \textbf{70.7{\color{red2}~$\uparrow$0.5}}\\
            \bottomrule
    \end{tabular}
        }
    \end{center}

        \caption{\textbf{Comparison of 3D Semantic Occupancy Prediction Using Single Frame on the Occ3D Dataset, Evaluated mIoU, mIoU\textsubscript{D}, and IoU Metrics.} Performance gains are indicated by red arrows {\color{red2}~$\uparrow$}. Our proposed approach (\textbf{+Ours}) consistently demonstrates superior enhancement over existing methods.}
   \label{tab:occ3d}
    \vspace{-4mm}
\end{table}

\subsection{Ablation Study}
\label{sec:ablation}
\noindent\textbf{Effect of Causal Loss.}
We first investigate the effectiveness of the proposed Causal Loss in enhancing the occupancy prediction performance. 
As demonstrated in \cref{tab:ablation}, the impact of the proposed Causal Loss is evaluated through a series of experiments. 
Specifically, using BEVDetOcc as the baseline (\textbf{Exp.~0}), we conducted two ablation studies: one removing the depth supervision loss (\textbf{Exp.~1}), and another incorporating the proposed Causal Loss (\textbf{Exps.~2, 3}). 
The results reveal that the removal of the depth supervision loss leads to marginal decrease in performance, whereas the addition of the proposed Causal Loss yields a significant improvement. 
This suggests that Causal Loss facilitates superior 2D-to-3D transformation and ultimately enhances the precision of 3D semantic occupancy prediction. Comparing \textbf{Exp.~3} to \textbf{Exp.~2}, the Unbiased Estimator simplifies the Causal Loss computation, reducing training overhead.

\noindent\textbf{Effect of Each Module.}
In \cref{tab:ablation} (\textbf{Exps. 5-8}), we systematically validate the effectiveness of the three proposed modules. We first remove the two post-lifting convolutional layers from the original BEVDet (\textbf{Exp. 4}), as they serve a similar role to our proposed modules in refining volume features. Subsequently, we incrementally integrate the three proposed modules (Channel-Grouped Lifting, Learnable Camera Offsets, and Normalized Convolution) into the baseline model. To ensure gradient flow for Learnable Camera Offsets, we introduce the Soft Filling strategy from ALOcc~\cite{chen2024alocc} in \textbf{Exp. 6}, enabling effective training of the camera parameter offsets in \textbf{Exp. 7}. The results show progressive performance improvements with each added component (\textbf{Exps. 5, 7, 8}), confirming the efficacy of the proposed SCAT method. Additionally, the proposed modules incur acceptable computational overhead.

 Please refer to the supplementary material for more comprehensive experimental results.

\begin{table}[t]
    \footnotesize
    \setlength{\abovecaptionskip}{0pt}
    \setlength{\tabcolsep}{0.01\linewidth}
    \begin{center}
    \resizebox{0.478\textwidth}{!}{
        \begin{tabular}{c|l|cc|c|c|c}
            \toprule
            Exp. & Method & mIoU & Diff. & mIoU\textsubscript{D} & IoU &Latency\\
            \midrule
            \midrule
            0 & Baseline~(BEVDetOcc) \cite{huang2022bevdet4d} & 37.1 & - & 30.2 & 70.4 &416/125\\
            \midrule
            1 & \textit{w/o} Depth Sup & 36.8 & {\color{blue}-0.3} & 29.6 & 70.3 &414/125\\
            \midrule
            \rowcolor{pink!8} 2 & + \textbf{\textit{Causal Loss}} & 37.6 & {\color{red2}+0.8} & 31.0 & 70.1 &450/125\\
            \rowcolor{pink!8} 3 & + \textbf{\textit{Unbiased Estimator}} & 37.5 & {\color{blue}-0.1} & 30.7 & 70.5 &417/125\\
            \midrule
            4 & \textit{w/o} Post Conv & 37.3  &{\color{blue}-0.2} & 30.7 & 70.2 &379/122\\
            \midrule
            \rowcolor{pink!8} 5 & + \textbf{\textit{Channel-Grouped Lifting}} & 37.6 & {\color{red2}+0.3} & 30.7 & 70.7 &419/128\\
            6 & + Soft Filling & 37.6 & {\color{blue}-} & 30.6 & 70.7 &434/149\\
            \rowcolor{pink!8} 7 & + \textbf{\textit{Learnable Camera Offset}} & 37.9 & {\color{red2}+0.3} & 31.1 & 71.0 &446/150\\
            \rowcolor{pink!8} 8 & + \textbf{\textit{Normalized Convolution}} & \textbf{38.3} & {\color{red2}+0.4} & \textbf{31.5} & \textbf{71.2} &466/159\\
       
            \bottomrule
        \end{tabular}
        }
    \end{center}
    \vspace{-0.5mm}
    \caption{\textbf{Ablation Study of 3D Semantic Occupancy Prediction on Occ3D.} We comprehensively evaluated the impact of the individual strategy (\textbf{bolded rows}) proposed in our paper, with BEVDetOcc as the baseline. The final column reports single-frame training/inference latency (ms) on an RTX 4090 GPU.
    }
    \label{tab:ablation}
    \vspace{-4.2mm}
\end{table}

\vspace{-0.5ex}
\section{Conclusion}
\vspace{-0.5ex}
In this paper, we introduced a novel approach leveraging causal principles to address the limitation that ignored the reliability and interpretability by existing methods. 
By exploring the causal foundations of 3D semantic occupancy prediction, we propose a causal loss that enhances semantic causal consistency.
In addition, we develop the SCAT module with three main components: Channel-Grouped Lifting, Learnable Camera Offsets and Normalized Convolution. 
This approach effectively mitigates transformation inaccuracies arising from uniform mapping weights, camera perturbations, and sparse mappings.
Experiments demonstrate that our approach achieves significant improvements in accuracy, robustness to camera perturbations, and semantic causal consistency in 2D-to-3D transformations.

{
    \small
    \bibliographystyle{ieeenat_fullname}
    \bibliography{main}
}

\end{document}